\documentclass[10pt,conference]{IEEEtran}
\IEEEoverridecommandlockouts

\usepackage{amsmath,amssymb,amsfonts}
\usepackage{algorithmic}
\usepackage{graphicx}
\usepackage{textcomp}
\usepackage{xcolor}
\usepackage{dblfloatfix}
\usepackage{placeins}
\usepackage{balance}

\makeatletter
\newcommand{\linebreakand}{
  \end{@IEEEauthorhalign}
  \hfill\mbox{}\par
  \mbox{}\hfill\begin{@IEEEauthorhalign}
}
\makeatother
\usepackage{etoolbox}
\usepackage{placeins}

\usepackage{graphicx}
\graphicspath{{./figures/}}

\usepackage[numbers]{natbib}
\usepackage{xspace}

\usepackage[x11names, rgb]{xcolor}
\usepackage{soul}

\definecolor{darkgreen}{HTML}{2C6E49}

\usepackage{subcaption} 
\usepackage{wrapfig}

\usepackage{tikz,tikz-3dplot}
\usepackage{tikzscale}
\usepackage{pgfplots}
\usepgfplotslibrary{colormaps}
\pgfplotsset{compat=1.18}

\usetikzlibrary{calc,fadings,decorations.pathmorphing,decorations.text,arrows,matrix,shapes,shapes.misc,patterns,plotmarks,positioning,intersections,3d,scopes,external}
\pgfdeclarelayer{back}
\pgfdeclarelayer{front}
\pgfsetlayers{back,main,front}

\tikzset{zlevel/.style={
    execute at begin scope={\pgfonlayer{#1}},
    execute at end scope={\endpgfonlayer}
}}

\tikzset{every picture/.style={font issue=\footnotesize},
    font issue/.style={execute at begin picture={#1\selectfont}}
}

\usepackage{amsmath}
\usepackage{amssymb}
\usepackage{amsthm}
\usepackage{physics}
\usepackage{mathtools}
\usepackage{nicefrac}
\usepackage{bm}
\usepackage{enumerate}
\usepackage{enumitem}
\newcounter{oldsection}
\newcounter{oldtheorem}
\theoremstyle{plain}

\makeatletter
\newcommand{\proofpart}[1]{
  \par
  \addvspace{\medskipamount}
  \noindent\emph{#1}\par\nobreak
  \addvspace{\smallskipamount}
  \@afterheading
}
\makeatother

\newtheorem{proposition}{Proposition} 

\newlength\newl
\newlength\newh
\newlength\colwidth
\usepackage{algorithm}
\usepackage{algorithmic}

\usepackage{comment}
\usepackage{csquotes}

\usepackage{array}
\usepackage{multirow}
\usepackage{tabularx}
\usepackage{booktabs}
\newcolumntype{C}[1]{>{\centering\arraybackslash}p{#1}}
\newcolumntype{L}[1]{>{\raggedright\arraybackslash}p{#1}}
\newcolumntype{R}[1]{>{\raggedleft\arraybackslash}p{#1}}

\usepackage[table]{xcolor} 
\usepackage{colortbl}      
\definecolor{Gray}{gray}{0.85}
\definecolor{NewGray}{gray}{.45} 
\definecolor{BlueGray}{rgb}{0.92, 0.92, 1}
\definecolor{LightCyan}{rgb}{0.88,1,1}
\definecolor{LightGreen}{rgb}{0.8, 0.99, 0.95}
\definecolor{DarkGreen}{rgb}{.1, .75, .1}
\definecolor{DarkRed}{rgb}{.95, .0, .1}
\definecolor{GrayGreen}{rgb}{0.6,0.85,0.6}
\definecolor{GrayRed}{rgb}{0.85,0.6,0.6}
\definecolor{bluey}{rgb}{0.0,0.6,0.6}

\newcommand{\cityscapes}{\textsc{Cityscapes}\xspace}
\newcommand{\voc}{\textsc{Pascal-Voc}\xspace}
\newcommand{\ade}{\textsc{Ade20K}\xspace}

\newcommand{\method}{TsallisPGD}
\newcommand{\acc}{\textsc{Acc}\xspace}
\newcommand{\miou}{mIoU\xspace}

\usepackage{url}
\usepackage{hyperref}
\hypersetup{
  colorlinks=true,
  urlcolor=blue,
  linkcolor=black,
  citecolor=black
}

\usepackage{amsmath,amsfonts,bm}

\def\eqref#1{equation~\ref{#1}}
\def\Eqref#1{Equation~\ref{#1}}

\def\1{\bm{1}}

\def\vx{{\bm{x}}}

\DeclareMathAlphabet{\mathsfit}{\encodingdefault}{\sfdefault}{m}{sl}
\SetMathAlphabet{\mathsfit}{bold}{\encodingdefault}{\sfdefault}{bx}{n}

\newcommand{\softmax}{\mathrm{softmax}}

\begin{document}

\setlength{\skip\footins}{4pt plus 1pt minus 1pt}

\title{TsallisPGD: Adaptive Gradient Weighting for Adversarial Attacks on Semantic Segmentation
	\thanks{Accepted to the 2026 IJCNN (author’s accepted version). \copyright~2026 IEEE.}}

\author{
	\IEEEauthorblockN{Alexander Matyasko, Xin Lou, Indriyati Atmosukarto, Wei Zhang}
	\IEEEauthorblockA{
		\textit{Singapore Institute of Technology}, Singapore\\
		\{alexander.matyasko, lou.xin, indriyati, wei.zhang\}@singaporetech.edu.sg
	}
}

\maketitle

\begin{abstract}
	Attacking semantic segmentation models is significantly harder than image classification models because an attacker must flip thousands of pixel predictions simultaneously.
	Standard pixel-wise cross-entropy (CE) is ill-suited to this setting: it tends to overemphasize already-misclassified pixels, which slows optimization and overstates model robustness.
	To address these issues, we introduce TsallisPGD, an adversarial attack built on the Tsallis cross-entropy, a generalization of CE parameterized by $q$, which adaptively reshapes the gradient landscape by controlling gradient concentration across pixels.
	By varying $q$, we steer the attack toward pixels at different confidence levels.
	We first show that no single fixed-$q$ is universally optimal, as its effectiveness depends on the dataset, model architecture, and perturbation budget.
	Motivated by this, we propose a dynamic $q$-schedule that sweeps $q$ during optimization.
	Extensive experiments on Cityscapes, Pascal VOC, and ADE20K show that TsallisPGD, using a single validation-selected schedule, achieves the best average attack rank across all evaluated settings and improves over CEPGD, SegPGD, CosPGD, JSPGD, and MaskedPGD in reducing accuracy and mIoU on both standard and robust models.
\end{abstract}

\section{Introduction}

Deep neural networks~(DNNs) have proven remarkably effective at tackling complex computer vision tasks, such as image classification~\cite{szegedy2016inception,he2015deep}, object detection~\cite{redmon2016yolo,liu2016ssd}, semantic segmentation~\cite{chen2018deeplab,xie2021segformer}, depth estimation~\cite{godard2019digging,bhat2021adabins}, and optical flow estimation~\cite{ranjan2017optical,teed2020raft}.
These advances have significant applications in autonomous driving~\cite{yurtsever2020survey} and biomedical image analysis~\cite{litjens2017survey}.
However, despite these impressive advancements, deep neural networks face critical challenges that limit their broader adoption in safety- and security-critical applications, such as self-driving cars and medical imaging.
Specifically, DNNs are sensitive to distribution shifts such as changes in lighting and weather conditions~\cite{taori2020measuring,recht2019dogenerailze}.
Most critically, state-of-the-art DNNs are vulnerable to adversarial attacks, which can craft small, imperceptible input perturbations that cause severe and unexpected model failures~\cite{szegedy2014intriguing,goodfellow2014explaining}.

Since the discovery of adversarial attacks against image classification models, significant progress has been made in evaluating and improving the robustness of image classification models~\cite{croce2020reliable, madry2017towards}.
Recent research has broadened to consider a wider range of prediction tasks~\cite{xie2017adversarial,thys2019fooling} and input modalities~\cite{jin2023plalidar, liang2018deeptext}.
Semantic segmentation presents unique challenges for adversarial robustness compared to image classification.
Although both types of attacks aim to optimize a loss function that measures attack success, adversarial attacks on semantic segmentation are significantly more challenging because they must alter predictions for all pixels simultaneously.
Likewise, improving the robustness of semantic segmentation models is more difficult, as any adversarial defense must ensure that predictions for all pixels remain robust to all feasible perturbations.

Early approaches for attacking semantic segmentation models~\cite{xie2017adversarial,arnab2018robustness} directly adapted the projected gradient descent (PGD) attack to the semantic segmentation setting.
These attacks maximize the average per-pixel loss, e.g. cross-entropy, across all output pixels, mirroring the training objective.
While maximizing average per-pixel loss is a straightforward extension of PGD, it ignores the multi-output nature of segmentation models, leading to suboptimal optimization dynamics.
As a result, PGD-based attacks on semantic segmentation models often suffer from slow convergence, requiring hundreds of attack iterations and potentially leading to an overly optimistic assessment of model robustness.

To address these limitations, recent works have focused on adaptively reweighting per-pixel losses to improve optimization dynamics.
SegPGD~\cite{gu2022segpgd} assigns different weights to correctly and incorrectly classified pixels.
SegPGD dynamically anneals pixel weights, initially focusing exclusively on correctly classified pixels and then gradually shifting attention to a mixture of correctly and incorrectly classified pixels.
CosPGD~\cite{agnihotri2024cospgd} sets pixel weights proportional to the cosine similarity between the original and adversarial outputs, effectively prioritizing pixels where the model is more confident.
Croce et al.~\cite{croce2025towards} further observed that gradients of standard cross-entropy saturate on misclassified pixels, causing attacks to overemphasize pixels that are already successfully attacked.
To mitigate this issue, they proposed alternative objectives with bounded gradients on misclassified pixels, such as the Jensen--Shannon divergence and masked cross-entropy, which yield more stable optimization dynamics.

In this paper, we propose \textbf{TsallisPGD}, an adversarial attack for semantic segmentation built on Tsallis cross-entropy, a generalization of standard cross-entropy controlled by a parameter~$q$.
A key property of Tsallis cross-entropy is that it induces a confidence-dependent reweighting of pixels in gradient space: each pixel is effectively scaled by $p_y^{1-q}$, where $p_y$ is the predicted probability of the ground-truth class.
For $q < 1$, this reshapes optimization by emphasizing high-confidence, correctly classified pixels while down-weighting high-confidence, incorrectly classified ones, making it well suited for segmentation attacks.
We analyze the resulting gradient behavior and show that $q$ controls both the magnitude and the location of the gradient peak.
By adjusting $q$, the attack can thus target different subsets of pixels.
We further show empirically that no single fixed-$q$ is optimal across datasets, architectures, and perturbation budgets.
This motivates a \emph{dynamic} $q$-schedule that sweeps $q$ during optimization enabling the attack to target pixels at different confidence levels.
We select the best $q$-schedule using a held-out validation split and apply the same selected schedule across the entire test benchmark with no further tuning.
Extensive experiments on \cityscapes, \voc, and \ade show that \method{} achieves the best overall average rank across both standard and adversarially trained models.
Our attack improves over CEPGD, SegPGD, CosPGD, JSPGD, and MaskedPGD in reducing both accuracy and mIoU on aggregate.
Code will be available at \url{https://github.com/aam-at/tsallis_pgd}.

To summarize, our main contributions are as follows:
\begin{itemize}
  \item We introduce TsallisPGD, a novel adversarial attack for semantic segmentation that leverages Tsallis cross-entropy to control the attack focus via parameter \( q \).
  \item We provide a theoretical analysis showing that Tsallis cross-entropy induces a confidence-dependent gradient reweighting $p_y^{1-q}$, and demonstrate how~$q$ controls both the location and magnitude of a lower-bound proxy for the gradient peak.
  \item We propose a dynamic $q$-schedule that varies $q$ over the course of optimization, enabling the attack to target pixels at different confidence levels.
  \item We demonstrate on \cityscapes, \voc, and \ade that \method{} achieves the best overall average rank among strong baselines on both standard and adversarially trained models.
\end{itemize}

\section{Related Work}

\textbf{Adversarial attacks against classification models:}
Seminal work by Szegedy et al.~\cite{szegedy2014intriguing} revealed that state-of-the-art image classifiers are highly vulnerable to adversarial perturbations—carefully crafted, small and often imperceptible input changes that can fool the model.
Goodfellow et al.~\cite{goodfellow2014explaining} introduced the Fast Gradient Sign Method (FGSM), a computationally efficient one-step attack that perturbs inputs in the direction of the gradient of the loss with respect to the input.
Kurakin et al.~\cite{kurakin2016adversarialexamples} and Madry et al.~\cite{madry2017towards} introduced stronger iterative attacks based on FGSM, such as BIM (I-FGSM) and PGD, respectively.
Madry et al.~\cite{madry2017towards} further formalized PGD as a universal first-order adversary, which has since become the de facto standard for adversarial robustness evaluation.
AutoPGD~\cite{croce2020reliable} improved upon PGD by detecting optimization stalling and adaptively tuning the step size of projected gradient iterations.
Despite these advances, reliably and accurately evaluating robustness remains challenging due to phenomena such as gradient masking and gradient obfuscation~\cite{athalye2018obfuscated}.
To address these issues, AutoAttack~\citep{croce2020reliable} combines multiple attacks into a unified, parameter-free evaluation suite and is now widely regarded as a standard benchmark for robustness evaluation in image classification.

\textbf{Adversarial attacks against semantic segmentation models:}
Early approaches for attacking semantic segmentation models~\cite{xie2017adversarial,arnab2018robustness} directly adapted classification attacks, such as projected gradient descent (PGD), by applying them in a pixel-wise manner and aggregating cross-entropy losses over all output pixels.
Although this formulation is a natural extension of PGD to semantic segmentation, it ignored the structured, multi-output nature of segmentation models, resulting in slow convergence and suboptimal optimization dynamics.
Consequently, these early works created the mistaken belief that segmentation models are inherently more robust than classification models~\cite{arnab2018robustness}.
However, later works identified that different pixels contribute unequally to the overall loss and showed that, when attacks are stronger and adaptive, segmentation models are as vulnerable as classifiers~\cite{gu2022segpgd}.
SegPGD~\cite{gu2022segpgd} improves attack effectiveness by adaptively reweighting pixel-wise losses based on whether pixels are correctly or incorrectly classified.
Using an adaptive annealing schedule, SegPGD achieves faster convergence and significantly outperforms standard PGD under the same iteration budget.
CosPGD~\cite{agnihotri2024cospgd} weights pixel losses according to the cosine similarity between the original and adversarial outputs, effectively prioritizing pixels where the model is more confident.
Further analyzing the shortcomings of standard objectives, Croce et al.~\cite{croce2025towards} observed that the gradient of the cross-entropy loss saturates on misclassified pixels, causing attacks to overemphasize pixels that have already been successfully attacked.
They introduced new segmentation attacks based on alternative loss functions, including JSPGD based on the Jensen--Shannon divergence~(JS) and MaskedPGD based on masked cross-entropy~(MaskedCE), whose gradients are bounded on already misclassified pixels and yield more stable optimization dynamics.
Similar to AutoAttack for evaluating classification models, they introduced SEA~\cite{croce2025towards}, a robustness evaluation suite that combines multiple attacks to provide a more reliable evaluation of adversarial robustness for semantic segmentation models.

\section{Semantic Segmentation Attacks}\label{sec:adversarial-attack}

Let $f_\theta$ be a semantic segmentation model with parameters $\theta$.
For notational convenience, we represent an input image as a vector $\vx \in [0,1]^D$, obtained by flattening an input image of spatial resolution $H \times W$ with $C$ channels, so that $D = C \cdot H \cdot W$.
The model maps $\vx$ to per-pixel unnormalized class scores or logits $f_\theta(\vx) \in \mathbb{R}^{K \times N}$, where $N = H \cdot W$ denotes the number of pixels (i.e., $K$ logits per pixel).
The predicted class at pixel $i$ is $\hat y_i(\vx) = \arg\max_k f_{k,i}(\vx)$.
Given a ground-truth segmentation map $y \in \{1,\dots,K\}^{N}$, an adversarial attack seeks a perturbation $\delta \in \mathcal{U}$ that degrades segmentation performance on the perturbed image $\vx' = \vx + \delta$, as measured by standard metrics such as pixel accuracy or mean intersection-over-union (mIoU).
In this work, we focus on $\ell_\infty$-bounded perturbations, i.e., $\mathcal{U} = \{\delta \in \mathbb{R}^{D} : \|\delta\|_\infty \le \varepsilon\}$.

Projected Gradient Descent (PGD)~\cite{madry2017towards} is a widely used first-order method for generating adversarial examples.
PGD approximately solves
\begin{equation}
  \max_{\lVert \delta \rVert_\infty \le \varepsilon} \; \mathcal{L}(\vx+\delta, y),
  \label{eq:pgd_obj}
\end{equation}
where $\mathcal{L}$ is a differentiable surrogate loss that measures attack success.
Starting from an initialization $\vx^{(0)}$, PGD iteratively takes a step in the direction of steepest ascent and projects the result onto the feasible set:
\begin{equation}
  \vx^{(t+1)} \;=\; \Pi_{\mathcal{B}_\infty(\vx,\varepsilon)}\!\Bigl(\vx^{(t)} + \alpha\, \mathrm{sign}\bigl(\nabla_\vx \mathcal{L}(\vx^{(t)},y)\bigr)\Bigr),
  \label{eq:pgd_update}
\end{equation}
where $\alpha$ is the step size and $\Pi_{\mathcal{B}_\infty(\vx,\varepsilon)}$ denotes projection onto the $\ell_\infty$ ball around $\vx$.

A natural choice for the surrogate loss $\mathcal{L}$ in semantic segmentation attacks is to use the sum or mean per-pixel cross-entropy, mirroring the standard training objective:
\begin{equation*}
  \mathcal{L}_{\mathrm{CE}}(\vx,y)
  \;=\;
  \frac{1}{N}\sum_{i=1}^{N} \ell_{\mathrm{CE}}\bigl(p_{y_i,i}(\vx)\bigr),
  \qquad
  \ell_{\mathrm{CE}}(p) = -\log p.
\end{equation*}

However, standard cross-entropy has several limitations when used as the attack loss for semantic segmentation, often leading to suboptimal optimization dynamics:
\paragraph{Objective monotonically increasing as $p_y \to 0$}
The per-pixel loss $\ell_{\mathrm{CE}}(p_{y_i,i})=-\log p_{y_i,i}$ increases monotonically as $p_{y_i,i} \to 0$ and is unbounded.
As a result, the average or mean loss can be dominated by a small subset of pixels with extremely small $p_{y_i,i}$, which are already misclassified, so improving the objective may largely reflect further decreasing $p_y$ on these pixels rather than improving the overall attack success.

\paragraph{Gradient monotonically increasing as $p_y \to 0$}
The per-pixel cross-entropy loss also leads to gradients that are highly concentrated on pixels with low predicted probability for the true class.
For the cross-entropy loss, its gradient with respect to logits is bounded as follows~(see Appendix~A~\cite{croce2025towards}):
\begin{equation}
\frac{K}{K-1}\,(1-p_y)^2
\;\le\;
\bigl\|\nabla_u \mathcal{L}_{\mathrm{CE}}(u,e_y)\bigr\|_2^2
\;\le\;
1 - p_y + (1-p_y)^2.
\end{equation}
Thus, as $p_y \to 0$, both the lower and upper bounds on gradients are monotonically increasing, resulting in overemphasis on already misclassified pixels and away from high-confidence correctly classified regions.

Due to these effects, PGD with cross-entropy loss struggles to change high-confidence pixels and overemphasizes low-confidence pixels, leading to suboptimal attack performance.
In the next section, we introduce a Tsallis cross-entropy objective, which is bounded and whose gradient peak can be adaptively controlled, enabling us to steer the attack and improve optimization dynamics.

\section{Gradient Rebalancing with Tsallis Cross-Entropy}\label{sec:tsallis-cross-entropy}

This section introduces Tsallis cross-entropy, a one-parameter generalization of standard cross-entropy controlled by a parameter $q$, known as an entropic index.
We show that Tsallis cross-entropy induces a confidence-dependent, per-pixel reweighting of the standard cross-entropy gradient, thereby rebalancing how optimization effort is distributed across pixels.
We analyze how the Tsallis parameter $q$ shapes the gradient.
We derive a closed-form expression for the gradient peak location, which provides an intuitive rule for selecting $q$ to target pixels within a desired confidence regime.
This improved control leads to more effective PGD optimization and higher attack success rates for semantic segmentation.

Let $p_{y_i,i} \in (0,1]$ denote the predicted probability of the ground-truth class $y_i$ at a pixel $i$.
For $q \in \mathbb{R}$ with $q\neq 1$, we define the Tsallis cross-entropy loss as
\begin{equation}
	\mathcal{L}_{q}(p_{y_i,i})
	\;=\;
	\frac{1 - p_{y_i,i}^{1-q}}{1-q}.
	\label{eq:tsallis_ce_def}
\end{equation}
Tsallis cross-entropy is bounded between $0$ and $\frac{1}{1-q}$ for all $q < 1$.
It reduces to the standard cross-entropy in the limit $q\to 1$.
For semantic segmentation attacks, we use the average of the per-pixel Tsallis losses.

\begin{proposition}[Tsallis cross-entropy is a confidence-weighted cross-entropy]
	The gradient of Tsallis cross-entropy with respect to $p$ is a confidence-weighted version of the standard cross-entropy gradient:
	\begin{equation}
		\frac{\partial \mathcal{L}_q(p)}{\partial p}
		\;=\;
		p^{1-q}\,\frac{\partial \mathcal{L}_{\mathrm{CE}}(p)}{\partial p}.
		\label{eq:tsallis_confidence_weighting}
	\end{equation}
\end{proposition}

\begin{proof}
	The gradient of the Tsallis loss $\mathcal{L}_q(p)$ from \eqref{eq:tsallis_ce_def} with respect to $p$ is computed as
	\begin{equation}
		\frac{\partial \mathcal{L}_q(p)}{\partial p}
		\;=\;
		-p^{-q}.
	\end{equation}
	On the other hand, the gradient of standard cross-entropy is
	\begin{equation}
		\frac{\partial \mathcal{L}_{\mathrm{CE}}(p)}{\partial p}
		\;=\;
		-\frac{1}{p}.
	\end{equation}
	Combining the two expressions gives
	\begin{equation}
		\frac{\partial \mathcal{L}_q}{\partial p}
		\;=\;
		-p^{-q}
		\;=\;
		p^{1-q}\left(-\frac{1}{p}\right)
		\;=\;
		p^{1-q}\,\frac{\partial \mathcal{L}_{\mathrm{CE}}(p)}{\partial p},
	\end{equation}
	which completes the proof of~\eqref{eq:tsallis_confidence_weighting}.
\end{proof}

When $q<1$, this weight increases with $p$ and therefore emphasizes high-confidence pixels; when $q>1$, it emphasizes low-confidence pixels.
For untargeted semantic segmentation attacks, we consider values of $q \in (-\infty,1)$ to focus optimization on high-confidence pixels that are difficult to flip.

The next proposition characterizes how Tsallis cross-entropy reshapes the gradient magnitude with respect to logits.

\begin{proposition}[Bounds for Tsallis Cross-Entropy Gradient]
	Let $u\in\mathbb{R}^K$ be the per-pixel logits, $p=\softmax(u)$, and let $y$ be the ground-truth class with probability $p_y\in(0,1]$.
	The gradient of Tsallis cross-entropy (per pixel) satisfies
	\begin{equation}
		\label{eq:tsallis_logit_grad_bound}
		\begin{split}
			\frac{K}{K-1}\,p_y^{2(1-q)}(1-p_y)^2
			 & \le \bigl\|\nabla_u \mathcal{L}_q(u,y)\bigr\|_2^2 \\
			 & \le p_y^{2(1-q)}\bigl(1-p_y+(1-p_y)^2\bigr).
		\end{split}
	\end{equation}
\end{proposition}
\Eqref{eq:tsallis_logit_grad_bound} can be derived by combining \eqref{eq:tsallis_confidence_weighting} with the standard cross-entropy gradient bounds.
For $q<1$, the factor $p_y^{2(1-q)}\le 1$ implies that the Tsallis gradient uniformly down-weights low-confidence pixels with $p_y \approx 0$, preventing gradients from concentrating on pixels that are already confidently misclassified.

\begin{proposition}[Tsallis Cross-Entropy Gradient Peak]
	For Tsallis cross-entropy with $q < 1$, the lower-bound proxy of the squared $\ell_2$-norm of the logit-gradient in~\eqref{eq:tsallis_logit_grad_bound} admits a unique maximum over $p_y\in(0,1)$ at
	\[
		p_y^\star \;=\; \frac{1-q}{2-q}.
	\]
\end{proposition}

\begin{proof}
	For simplicity, we analyze the lower bound in \eqref{eq:tsallis_logit_grad_bound}.
	Since constant factors do not affect the maximizer, it suffices to maximize $f(p)=p^{2(1-q)}(1-p)^2$ over $p\in(0,1)$.
	Maximizing $\log f(p)=2(1-q)\log p + 2\log(1-p)$ yields the following first-order condition:
	\[
		\frac{d}{dp}\log f(p)=\frac{2(1 - q)}{p}-\frac{2}{1-p}=0
		\, \Longrightarrow \,
		p^\star=\frac{1 - q}{2 -q}.
	\]
\end{proof}

The expression for $p_y^\star$ provides an intuitive mechanism for controlling gradient concentration.
Smaller values of $q$ shift the gradient peak toward high-confidence, correctly classified pixels, while increasing $q$ moves the peak toward lower-confidence regions, recovering standard cross-entropy behavior as $q\to 1$.
Figure~\ref{fig:gradient_weighting_comparison} compares the per-pixel confidence-dependent gradients induced by Tsallis cross-entropy with those of standard cross-entropy, cross-entropy combined with cosine similarity~\cite{agnihotri2024cospgd}, and Jensen–Shannon divergence~\cite{croce2025towards}.
Unlike prior approaches, Tsallis cross-entropy allows the location of the gradient peak to be explicitly controlled through the parameter $q$, thereby enabling precise control over which pixels the attack focuses on.
How $q$ should be chosen in practice is an empirical question.
In Section~\ref{sec:fixed_q_study}, we show that the best fixed-$q$ varies across datasets, architectures, and threat budgets.
In Section~\ref{sec:schedule_selection}, based on this observation, we propose a dynamic $q$-schedule and select the best schedule on a held-out validation split.
Our schedule starts with $q < 0$ to emphasize high-confidence, correctly classified pixels in early iterations, and gradually increases $q \to 1$, recovering standard cross-entropy by the end of optimization.

\begin{figure}[t]
	\centering
	\begin{tikzpicture}
\begin{axis}[
    width=\columnwidth,
    height=0.68\columnwidth,
    xmin=0, xmax=1,
    ymin=0, ymax=1.08,
    domain=0.001:0.999,
    samples=180,
    xlabel={$p_y$},
    ylabel={normalized gradient weight},
    axis lines=left,
    axis line style={thick,-stealth},
    tick style={thick},
    ticklabel style={font=\footnotesize},
    label style={font=\small},
    xtick={0,0.2,0.4,0.6,0.8,1},
    ytick={0,0.2,0.4,0.6,0.8,1},
    ymajorgrids,
    grid style={black!10},
    legend columns=2,
    legend style={
        draw=none,
        fill=none,
        font=\scriptsize,
        at={(0.5,-0.22)},
        anchor=north,
        cells={anchor=west},
        column sep=1.2em,
        row sep=0.2em
    },
    legend image code/.code={
        \draw[#1, line width=1.4pt] (0cm,0cm) -- (0.75cm,0cm);
    },
]

\addplot[
    orange!85!black,
    line width=1.4pt,
] {1 - x};
\addlegendentry{Cross-Entropy}

\addplot[
    magenta!70!black,
    dashed,
    line width=1.4pt,
] {sin(deg(3.141592653589793*x))};
\addlegendentry{CosPGD}

\addplot[
    green!60!black,
    dashdotted,
    line width=1.4pt,
] {(x^0.35)*(1-x)/0.46182052223806414};
\addlegendentry{Jensen--Shannon}

\addplot[
    blue!80!black,
    densely dotted,
    line width=1.4pt,
] {(x^1.5)*(1-x)/0.185903200617956};
\addlegendentry{Tsallis CE $q=-0.5$}

\addplot[
    blue!80!black,
    dashed,
    line width=1.4pt,
] {x*(1-x)/0.25};
\addlegendentry{Tsallis CE $q=0$}

\addplot[
    blue!45!black,
    dashdotted,
    line width=1.4pt,
] {sqrt(x)*(1-x)/0.38490017945975047};
\addlegendentry{Tsallis CE $q=0.5$}

\end{axis}
\end{tikzpicture}
	\caption{\textbf{Comparison of per-pixel gradient weightings.}
		Best viewed in color.
		We compare per-pixel gradient weightings for different attacks as a function of the predicted probability $p_y$ of the ground-truth class.
		Varying $q$ in Tsallis cross-entropy shifts the location of the gradient peak, enabling targeted optimization over pixels at different confidence levels.}
	\label{fig:gradient_weighting_comparison}
\end{figure}
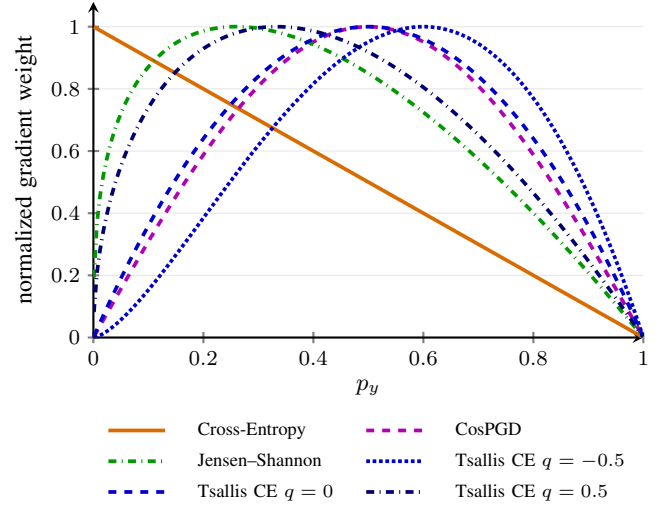

\section{Experiments}
\label{sec:experiments}

\begin{table*}[t!]
	\caption{\textbf{Effect of $q$ on attack strength.}
		We evaluate TsallisPGD with fixed-$q$ on the strongest robust model for each dataset:
		PSPNet (DDCAT) on Cityscapes, UPerNet (ConvNeXt-T) on VOC, and Segmenter (ViT-S) on ADE20K.
		\emph{Best-of-$q$} reports the lowest value obtained by aggregating multiple fixed-$q$ variants using SEA~\cite{croce2025towards}, representing the strongest combination of fixed-$q$ attacks.
		All values are rounded to one decimal place; attacks tied at the rounded best value are all marked best.
		Lower values indicate stronger attacks.
		For each row and metric, the best fixed-$q$ result is shown in \textbf{bold}.}
	\label{tab:q_analysis}

	\centering
	\footnotesize
	\setlength{\tabcolsep}{4.5pt}
	\renewcommand{\arraystretch}{1.05}

	\resizebox{\textwidth}{!}{
		\begin{tabular}{c|cc|cc|cc|cc|cc|cc||cc}
			\multirow{2}{*}{$\epsilon_\infty$}
			         & \multicolumn{12}{c||}{Fixed-$q$}
			         & \multicolumn{2}{c}{Best-of-$q$}                                                                                                                                                                                       \\

			         & \multicolumn{2}{c|}{$q=-3$}
			         & \multicolumn{2}{c|}{$q=-2$}
			         & \multicolumn{2}{c|}{$q=-1$}
			         & \multicolumn{2}{c|}{$q=0$}
			         & \multicolumn{2}{c|}{$q=0.5$}
			         & \multicolumn{2}{c||}{$q=1$}
			         & \multicolumn{2}{c}{}                                                                                                                                                                                                  \\

			         & Acc                              & mIoU & Acc           & mIoU          & Acc           & mIoU          & Acc           & mIoU          & Acc           & mIoU          & Acc  & mIoU & Acc           & mIoU          \\
			\midrule

			\multicolumn{15}{l}{\textbf{\cityscapes:} PSPNet+ResNet-50 \textbf{(DDCAT~\cite{xu2021dynamic})} \quad (Acc 94.9, mIoU 64.7)}                                                                                                    \\
			\midrule
			0.25/255 & 63.6                             & 34.9 & \textbf{63.2} & 33.8          & 63.3          & 33.0          & 64.8          & \textbf{32.5} & 66.6          & 33.3          & 72.3 & 35.3 & \textbf{61.9} & \textbf{30.4} \\
			0.5/255  & 28.6                             & 16.1 & 28.1          & 15.7          & \textbf{27.8} & \textbf{14.7} & 28.9          & 14.9          & 31.2          & 15.3          & 43.0 & 21.0 & \textbf{24.6} & \textbf{12.4} \\
			1/255    & 5.4                              & 3.4  & 4.5           & 2.8           & \textbf{4.1}  & 2.4           & 4.2           & \textbf{2.0}  & 5.8           & 2.6           & 17.8 & 7.2  & \textbf{2.0}  & \textbf{1.1}  \\
			\midrule
			\midrule

			\multicolumn{15}{l}{\textbf{\voc:} UPerNet+ConvNeXt-T \textbf{(PIR-AT~\cite{croce2025towards})} \quad (Acc 92.7, mIoU 75.2)}                                                                                                     \\
			\midrule
			4/255    & 90.0                             & 67.7 & 89.7          & 67.0          & 89.3          & 66.1          & 88.9          & 65.2          & \textbf{88.8} & \textbf{64.9} & 89.0 & 65.5 & \textbf{88.6} & \textbf{64.4} \\
			8/255    & 76.4                             & 43.5 & 75.1          & 41.9          & 73.9          & 40.5          & \textbf{73.1} & \textbf{40.0} & 73.9          & 41.4          & 77.4 & 46.1 & \textbf{71.2} & \textbf{36.3} \\
			12/255   & 32.3                             & 10.6 & 31.6          & \textbf{10.2} & \textbf{31.5} & \textbf{10.2} & 34.2          & 12.1          & 38.5          & 15.0          & 52.4 & 22.9 & \textbf{27.2} & \textbf{7.3}  \\
			\midrule
			\midrule

			\multicolumn{15}{l}{\textbf{\ade:} Segmenter+ViT-S \textbf{(PIR-AT~\cite{croce2025towards})} \quad (Acc 69.1, mIoU 28.6)}                                                                                                        \\
			\midrule
			4/255    & 59.4                             & 20.3 & 58.5          & 19.6          & 57.4          & 18.5          & 56.1          & 17.3          & \textbf{55.8} & \textbf{16.8} & 57.1 & 17.3 & \textbf{55.2} & \textbf{15.9} \\
			8/255    & 40.3                             & 10.9 & 38.5          & 10.0          & 36.2          & 9.0           & \textbf{34.5} & 8.0           & 34.9          & \textbf{7.8}  & 39.9 & 9.3  & \textbf{32.9} & \textbf{6.5}  \\
			12/255   & 11.7                             & 2.8  & 10.8          & 2.4           & \textbf{9.8}  & 2.2           & 10.2          & \textbf{2.1}  & 11.6          & 2.3           & 20.3 & 4.0  & \textbf{8.3}  & \textbf{1.4}  \\

			\bottomrule
		\end{tabular}
	}
\end{table*}

We evaluate and compare \method{} against state-of-the-art segmentation attacks.
Our evaluation proceeds in three stages to motivate and validate the dynamic $q$-schedule.
First, in Section~\ref{sec:fixed_q_study}, we study \emph{fixed-$q$} and show that no single value is universally optimal: the best $q$ depends on the dataset, architecture, and threat budget.
This motivates a dynamic schedule that sweeps $q$ within a single attack.
In Section~\ref{sec:schedule_selection}, we define a linear $q$-schedule parameterized by $(q_{\text{start}}, q_{\text{end}})$ and select these parameters using a held-out validation split, disjoint from the test set used in the final comparison.
Finally, in Section~\ref{sec:main_comparison}, we fix this validation-selected schedule and evaluate \method{} against state-of-the-art segmentation attacks on the test set.

\noindent \textbf{Datasets:}
We conduct experiments on \voc~\cite{everingham2010pascal}, \ade~\cite{zhou2019ade20k}, and \cityscapes~\cite{cordts2016cityscapes}, following standard protocols and evaluating on the central crop~\cite{croce2025towards}.

\noindent \textbf{Models:}
We evaluate diverse segmentation architectures, both standard and adversarially trained: PSPNet~\cite{zhao2017pyramid} (ResNet-50 backbone) with clean and DDCAT-robust weights, UPerNet~\cite{xiao2018unified} with a ConvNeXt-T backbone~\cite{liu2022convnet}, and Segmenter~\cite{strudel2021segmenter} with a ViT-S backbone~\cite{dosovitskiy2021image}.

\noindent \textbf{Threat model and metrics:}
We consider the $\ell_\infty$ threat model with perturbation radius $\epsilon_\infty$ (reported in the tables).
We report (i) average pixel accuracy (\acc) and (ii) mean intersection-over-union (\miou), where lower values indicate stronger attacks.

\noindent \textbf{Baselines and protocol:}
We compare \method{} against strong PGD-based objectives for segmentation: CEPGD (pixel-wise cross-entropy), SegPGD~\cite{gu2022segpgd}, CosPGD~\cite{agnihotri2024cospgd}, JSPGD~\cite{croce2025towards} (Jensen--Shannon divergence PGD), and MaskedPGD~\cite{croce2025towards} (masked cross-entropy PGD).
All attacks, including TsallisPGD, use APGD~\cite{croce2020reliable} for step-size selection, a fixed budget of $T{=}300$ iterations, a single random start, and the multi-$\varepsilon$ trick from~\cite{croce2025towards}.
This ensures that performance differences are attributable to the attack objective.

\subsection{\texorpdfstring{Analysis of Fixed-$q$ TsallisPGD}{Analysis of Fixed-q TsallisPGD}}\label{sec:fixed_q_study}

The Tsallis objective (Section~\ref{sec:tsallis-cross-entropy}) introduces a continuous parameter $q$ that controls the concentration of per-pixel gradients.
A natural question is whether a single fixed value of $q$ is sufficient, or whether stronger attacks require varying $q$ during optimization.
To study this, we evaluate \method{} with constant $q \in \{-3,-2,-1,0,0.5,1\}$, which correspond to gradient peak locations ranging from $0.8$ to $0.0$, on the strongest robust model for each dataset: PSPNet (DDCAT~\cite{xu2021dynamic}) on Cityscapes, UPerNet (ConvNeXt-T, PIR-AT~\cite{croce2025towards}) on Pascal VOC, and Segmenter (ViT-S, PIR-AT~\cite{croce2025towards}) on ADE20K.
Here, $q=1$ recovers the standard cross-entropy objective.
We also report the \emph{Best-of-$q$} result obtained by aggregating several fixed-$q$ variants using SEA~\cite{croce2025towards}, which represents the best combination of multiple fixed-$q$ attacks.

Table~\ref{tab:q_analysis} shows that no single fixed value of $q$ is optimal across datasets, models, perturbation budgets, or even evaluation metrics.
For example, on Cityscapes, the lowest accuracy is achieved by $q=-2$ at $\epsilon_\infty=0.25/255$, but by $q=-1$ at $0.5/255$ and $1/255$.
On the other hand, the best mIoU is obtained with $q=0$ at $0.25/255$ and $1/255$, and with $q=-1$ at $0.5/255$.
The table also shows that standard CEPGD ($q=1$) is consistently suboptimal.
On VOC at $12/255$, accuracy drops from $52.4$ with $q=1$ to $31.5$ with the best fixed choice of $q$.
These results confirm that replacing standard cross-entropy with Tsallis cross-entropy, and tuning $q$, can substantially strengthen the attack.

Finally, the \emph{Best-of-$q$} baseline is stronger than the best single fixed-$q$ attack in every row, and often yields substantially stronger results.
This suggests that different values of $q$ expose different vulnerable subsets of pixels, so no single fixed choice can capture all of them.
This directly motivates a \emph{linear} dynamic $q$-schedule: by sweeping $q$ during optimization, a single attack run can traverse multiple gradient-concentration regimes and target pixels at different confidence levels.
In the next subsection, we describe how we choose the endpoints of this linear schedule, $(q_{\text{start}}, q_{\text{end}})$.

\subsection{\texorpdfstring{Selection of the Dynamic $q$-Schedule}{Selection of the Dynamic q-Schedule}}\label{sec:schedule_selection}

The observation that the best fixed value of $q$ varies across settings motivates sweeping $q$ during optimization, allowing the attack to focus on pixels at different confidence levels throughout adversarial generation.
This removes the need to tune $q$ separately for each dataset, model, and perturbation budget and, when chosen appropriately, can approach the performance of the best combination of fixed-$q$ attacks.

We consider a linear schedule parameterized by the range $[q_{\text{start}}, q_{\text{end}}]$.
Specifically, we evaluate schedules with $q_{\text{start}} \in \{-3,-2,-1\}$ and $q_{\text{end}} \in \{0.5,1.0\}$, using linear interpolation, for a total of six candidate schedules.
Each candidate is evaluated on the \emph{validation} split of the corresponding dataset, following the protocol in Section~\ref{sec:fixed_q_study}.
For each schedule, we compute its rank on every validation row and then average these ranks across all rows.
The candidate with the best average validation rank is the linear sweep of $q$ over $[-2,1]$.
We therefore select this schedule as the \method{} schedule and use it unchanged for the final benchmark in Section~\ref{sec:main_comparison} across all datasets, architectures, and perturbation budgets.

\subsection{Comparison with State-of-the-Art Attacks}\label{sec:main_comparison}

\begin{table*}[!th]
	\caption{\textbf{Comparison of attacks.}
		We evaluate all attacks for 300 iterations on clean and robust models trained on \voc, \ade, and \cityscapes.
		\method{} uses the dynamic $q$-schedule selected in Section~\ref{sec:schedule_selection}, applied identically across all rows without per-row tuning.
		We report average pixel accuracy (Acc) and mean IoU (mIoU), where lower values indicate stronger attacks.
		Clean performance is shown next to each model.
		Attacks are ranked independently for each row and metric, with lower ranks indicating stronger attacks; \textbf{Avg. Rank} reports the mean rank across all rows, datasets, models, and $\epsilon_\infty$ budgets.
		For each metric, the best result is shown in \textbf{bold} and the second best is \underline{underlined}.}
	\label{tab:main_results}

	\centering
	\footnotesize
	\setlength{\tabcolsep}{3.65pt}
	\renewcommand{\arraystretch}{1.05}

	\resizebox{\textwidth}{!}{
		\begin{tabular}{c|cc|cc|cc|cc|cc||cc}
			\multirow{2}{*}{$\epsilon_\infty$}
			                   & \multicolumn{10}{c||}{Previous SOTA attacks}
			                   & \multicolumn{2}{c}{Our attack}                                                                                                                                                                                                                    \\

			                   & \multicolumn{2}{c|}{CEPGD}
			                   & \multicolumn{2}{c|}{SegPGD}
			                   & \multicolumn{2}{c|}{CosPGD}
			                   & \multicolumn{2}{c|}{JSPGD}
			                   & \multicolumn{2}{c||}{MaskedPGD}
			                   & \multicolumn{2}{c}{TsallisPGD}                                                                                                                                                                                                                    \\

			                   & Acc                                          & mIoU & Acc              & mIoU             & Acc              & mIoU             & Acc              & mIoU             & Acc              & mIoU             & Acc              & mIoU             \\
			\midrule

			\multicolumn{13}{l}{\textbf{\cityscapes:} PSPNet+ResNet-50 \textbf{(clean training)} \quad (Acc 95.3, mIoU 68.6)}                                                                                                                                                      \\
			\midrule
			0.25/255           & 22.1                                         & 12.9 & 14.1             & 9.2              & 7.9              & 5.8              & 9.0              & 6.5              & \textbf{6.5}     & \textbf{5.0}     & \underline{7.2}  & \underline{5.6}  \\
			0.5/255            & 5.3                                          & 1.5  & 1.0              & 0.3              & \underline{0.5}  & \underline{0.2}  & 0.7              & \underline{0.2}  & \textbf{0.2}     & \textbf{0.1}     & \textbf{0.2}     & \textbf{0.1}     \\
			1/255              & 3.9                                          & 1.0  & 0.5              & 0.3              & 0.2              & \underline{0.1}  & \underline{0.1}  & \underline{0.1}  & \textbf{0.0}     & \textbf{0.0}     & \textbf{0.0}     & \textbf{0.0}     \\
			\midrule

			\multicolumn{13}{l}{\textbf{\cityscapes:} PSPNet+ResNet-50 \textbf{(DDCAT~\cite{xu2021dynamic})} \quad (Acc 94.9, mIoU 64.7)}                                                                                                                                          \\
			\midrule
			0.25/255           & 72.0                                         & 34.9 & 66.9             & 34.1             & 64.7             & \textbf{32.1}    & 66.0             & \underline{32.4} & \underline{63.7} & 32.8             & \textbf{62.9}    & 32.8             \\
			0.5/255            & 42.2                                         & 20.0 & 36.2             & 18.3             & \underline{27.1} & \underline{14.0} & 29.2             & 14.3             & \textbf{26.2}    & \textbf{13.9}    & 27.3             & 14.1             \\
			1/255              & 16.2                                         & 6.4  & 12.4             & 5.8              & \underline{3.0}  & \textbf{1.5}     & 3.8              & 1.8              & \textbf{2.7}     & \underline{1.6}  & 4.1              & 2.1              \\
			\midrule

			\multicolumn{13}{l}{\textbf{\voc:} PSPNet+ResNet-50 \textbf{(clean training)} \quad (Acc 91.9, mIoU 85.9)}                                                                                                                                                             \\
			\midrule
			0.25/255           & 48.1                                         & 22.1 & 36.5             & 14.4             & 33.7             & 13.9             & 36.5             & 16.1             & \underline{32.5} & \underline{13.1} & \textbf{32.4}    & \textbf{12.9}    \\
			0.5/255            & 23.1                                         & 11.2 & 8.3              & 3.3              & 5.6              & 2.7              & 7.7              & 4.2              & \textbf{4.2}     & \textbf{1.5}     & \underline{4.3}  & \underline{1.7}  \\
			1/255              & 13.2                                         & 7.3  & 1.1              & 0.9              & \underline{0.2}  & \underline{0.3}  & 0.7              & 0.8              & \textbf{0.1}     & \textbf{0.1}     & \textbf{0.1}     & \textbf{0.1}     \\
			\midrule

			\multicolumn{13}{l}{\textbf{\voc:} PSPNet+ResNet-50 \textbf{(DDCAT~\cite{xu2021dynamic})} \quad (Acc 94.1, mIoU 76.9)}                                                                                                                                                 \\
			\midrule
			0.25/255           & 72.3                                         & 36.8 & 68.3             & 33.8             & \underline{67.3} & \textbf{33.1}    & 68.2             & 33.6             & \textbf{67.0}    & 33.4             & \textbf{67.0}    & \underline{33.2} \\
			0.5/255            & 46.5                                         & 18.0 & 33.8             & 12.5             & 30.3             & 11.6             & 33.1             & 12.9             & \textbf{27.7}    & \underline{10.6} & \underline{28.1} & \textbf{10.3}    \\
			1/255              & 22.7                                         & 8.3  & 10.1             & 3.8              & 3.1              & \underline{1.5}  & 5.0              & 2.8              & \underline{2.4}  & \textbf{1.1}     & \textbf{2.3}     & \textbf{1.1}     \\
			\midrule

			\multicolumn{13}{l}{\textbf{\voc:} UPerNet+ConvNeXt-T \textbf{(PIR-AT~\cite{croce2025towards})} \quad (Acc 92.7, mIoU 75.2)}                                                                                                                                           \\
			\midrule
			4/255              & 89.0                                         & 65.5 & \textbf{88.7}    & \textbf{64.7}    & \underline{88.8} & 65.0             & \textbf{88.7}    & \underline{64.8} & 89.2             & 65.8             & \textbf{88.7}    & \textbf{64.7}    \\
			8/255              & 77.2                                         & 45.8 & \underline{73.0} & \underline{39.1} & 73.1             & 40.1             & 74.0             & 41.5             & 73.8             & 40.4             & \textbf{72.8}    & \textbf{39.0}    \\
			12/255             & 53.1                                         & 23.2 & 35.9             & 11.6             & 34.6             & 12.6             & 38.4             & 15.2             & \underline{31.5} & \underline{10.1} & \textbf{30.7}    & \textbf{9.7}     \\
			\midrule

			\multicolumn{13}{l}{\textbf{\ade:} UPerNet+ConvNeXt-T \textbf{(PIR-AT~\cite{croce2025towards})} \quad (Acc 70.5, mIoU 31.8)}                                                                                                                                           \\
			\midrule
			4/255              & 57.5                                         & 19.8 & \underline{55.9} & \textbf{18.9}    & \underline{55.9} & \underline{19.0} & \underline{55.9} & \underline{19.0} & 57.0             & 20.0             & \textbf{55.8}    & \textbf{18.9}    \\
			8/255              & 35.2                                         & 8.9  & 28.4             & 7.2              & \underline{27.6} & \textbf{7.0}     & 28.4             & 7.3              & 28.5             & 7.5              & \textbf{27.5}    & \underline{7.1}  \\
			12/255             & 13.2                                         & 2.5  & 5.3              & 1.2              & \underline{3.9}  & \underline{1.0}  & 4.8              & 1.2              & \textbf{3.4}     & \textbf{0.8}     & \textbf{3.4}     & \textbf{0.8}     \\
			\midrule

			\multicolumn{13}{l}{\textbf{\ade:} Segmenter+ViT-S \textbf{(PIR-AT~\cite{croce2025towards})} \quad (Acc 69.1, mIoU 28.6)}                                                                                                                                              \\
			\midrule
			4/255              & 57.0                                         & 17.3 & \underline{55.5} & \textbf{16.5}    & 55.7             & 16.8             & 55.6             & \underline{16.6} & 56.8             & 17.8             & \textbf{55.4}    & \textbf{16.5}    \\
			8/255              & 39.8                                         & 9.4  & \underline{34.3} & \textbf{7.5}     & \textbf{34.0}    & \underline{7.6}  & 34.4             & 7.7              & 36.0             & 8.5              & \textbf{34.0}    & \underline{7.6}  \\
			12/255             & 19.8                                         & 3.9  & 10.4             & \textbf{2.0}     & \underline{9.6}  & \textbf{2.0}     & 10.7             & \underline{2.1}  & 9.8              & \underline{2.1}  & \textbf{9.1}     & \textbf{2.0}     \\

			\midrule
			\addlinespace[1mm]
			\textbf{Avg. Rank} & 5.95                                         & 5.86 & 4.05             & 3.52             & 2.81             & \underline{2.57} & 3.71             & 3.86             & \underline{2.43} & 2.76             & \textbf{1.38}    & \textbf{1.57}    \\
			\bottomrule
		\end{tabular}
	}
\end{table*}

With the $q$-schedule fixed by the selection procedure in Section~\ref{sec:schedule_selection}, we compare \method{} against state-of-the-art segmentation attacks on the full test benchmark.
The same schedule is used unchanged across all datasets, architectures, and perturbation budgets, with no per-row tuning.

\noindent \textbf{Results:}
Table~\ref{tab:main_results} shows that \method{} is the strongest attack on average, achieving the best average rank across all 21 evaluation settings: $1.38$ in \acc{} and $1.57$ in \miou{}.
For \acc{}, this improves over the strongest baseline, MaskedPGD ($2.43$), while for \miou{} it outperforms CosPGD ($2.57$), establishing consistent gains across both metrics.
At the per-setting level, \method{} attains the best performance in $16/21$ settings for \acc{} and $13/21$ settings for \miou{} (counting ties as best), indicating that its advantage is not limited to a small subset of configurations but holds broadly.

The gains are most pronounced on adversarially trained models, where optimization is typically more challenging.
On VOC with UPerNet (ConvNeXt-T, PIR-AT), \method{} achieves the best \acc{} at all three perturbation budgets and is best (or tied) in \miou{} across all budgets.
On ADE20K with UPerNet (ConvNeXt-T, PIR-AT), it is again best in \acc{} at all budgets and best (or tied) in \miou{} at $4/255$ and $12/255$, while at $8/255$ it is narrowly second to CosPGD ($7.1$ vs.\ $7.0$).
Similarly, on ADE20K with Segmenter (ViT-S, PIR-AT), \method{} is strongest at $4/255$ and $12/255$ on both metrics; at $8/255$ it achieves the best \acc{} but is marginally second in \miou{}, where SegPGD performs best.
On clean models, the results are more mixed but remain favorable: \method{} is competitive across Cityscapes and VOC and achieves several clear wins, although MaskedPGD occasionally matches or slightly exceeds its performance at larger perturbation budgets.
This behavior suggests that Tsallis-based reweighting is especially effective in more challenging robustness-evaluation regimes against adversarially trained models, where controlling the gradient focus across pixels becomes central to attack optimization.
Overall, the results indicate that a single validation-selected dynamic Tsallis schedule generalizes across datasets, architectures, and perturbation budgets, outperforming prior state-of-the-art segmentation attacks.
A natural direction for future work is to make the choice of \(q\) adaptive, for example by selecting \(q\) automatically at each attack iteration rather than relying on a fixed, predetermined schedule.

\section{Conclusion}
\label{sec:conclusion}

We analyze the limitations of standard cross-entropy as an attack objective for semantic segmentation and introduce \textbf{TsallisPGD}, which replaces pixel-wise cross-entropy with Tsallis cross-entropy.
We show that the parameter $q$ controls which confidence levels the attack emphasizes, allowing it to target different subsets of pixels during optimization.
Because the best choice of $q$ depends on the dataset, model, and perturbation budget $\epsilon_\infty$, no single fixed value is optimal across settings.
This motivates a \emph{dynamic} $q$-schedule that varies $q$ throughout the attack.
Using a held-out validation set, we select the best linear schedule, which sweeps $q$ from $-2$ to $1$.
Across \cityscapes, \voc, and \ade, \method{} with this schedule achieves the best overall average rank among CEPGD, SegPGD, CosPGD, JSPGD, and MaskedPGD in both global accuracy and mIoU.
These results show that controlling gradient concentration through Tsallis cross-entropy provides a simple and general way to construct stronger adversarial attacks for semantic segmentation.
In future work, we plan to make the choice of $q$ more adaptive and incorporate TsallisPGD into adversarial training to improve robustness.

\section*{Acknowledgment}
This research is supported by the National Research Foundation, Singapore, under its AI Singapore Programme (AISG Award No: AISG4-GC-2023-006-1B), the Ministry of Education, Singapore, under its Academic Research Tier-1 Grant (Award No: 1091/R-MA124-R205-0002), A*STAR under its MTC Individual Research Grant (Award No: M23M6c0113), and MTC Programmatic Grant (Award No: M23L9b0052).

\balance
\bibliographystyle{IEEEtran}
\bstctlcite{IEEEexample:BSTcontrol}
\bibliography{refs_short_abbr}

\end{document}